\documentclass{article}

\usepackage{microtype}
\usepackage{graphicx}
\usepackage{subcaption}
\usepackage{booktabs} 

\usepackage{hyperref}

\usepackage[preprint]{icml2026}

\usepackage{amsmath}
\usepackage{amssymb}
\usepackage{mathtools}
\usepackage{amsthm}
\usepackage{multirow}
\usepackage[table,xcdraw]{xcolor}

\usepackage[capitalize,noabbrev]{cleveref}

\theoremstyle{plain}

\newtheorem{lemma}{Lemma}

\theoremstyle{definition}

\newtheorem{assumption}{Assumption}
\theoremstyle{remark}

\usepackage[textsize=tiny]{todonotes}

\icmltitlerunning{Difficulty-Differentiated Policy Optimization
with Length Redistribution for Efficient and Robust Reinforcement Learning}

\begin{document}

\twocolumn[
  \icmltitle{Balancing the Reasoning Load:
  Difficulty-Differentiated Policy Optimization \\ with Length 
  Redistribution for Efficient and Robust Reinforcement Learning}



  \icmlsetsymbol{equal}{*}

  \begin{icmlauthorlist}
    \icmlauthor{Yinan Xia}{xxx}
    \icmlauthor{Haotian Zhang}{yyy}
    \icmlauthor{Huiming Wang}{yyy}
  \end{icmlauthorlist}

  \icmlaffiliation{yyy}{Kuaishou Technology }
  \icmlaffiliation{xxx}{Tianjin University}

  \icmlcorrespondingauthor{Huiming Wang}{huiming\_wang@mymail.sutd.edu.sg}

  \icmlkeywords{Machine Learning, ICML}

  \vskip 0.3in
]



\printAffiliationsAndNotice{}  

\begin{abstract}
  Large Reasoning Models (LRMs) have shown exceptional reasoning capabilities, but they also suffer from the issue of \textit{overthinking}, often generating excessively long and redundant answers.
  For problems that exceed the model's capabilities, LRMs tend to exhibit the \textit{overconfidence} phenomenon, generating overly short but incorrect answers, which may contribute to suboptimal performance. 
  To address these issues, we propose Difficulty-Differentiated Policy Optimization (DDPO), an efficient reinforcement learning algorithm that optimizes simple and complex tasks separately based on the \textit{overconfidence} phenomenon.
  Specifically, it reduces the output length for simple tasks without compromising accuracy, while for complex tasks, it expands the exploration space to improve performance. We further derive the theoretical conditions for maximizing expected accuracy, which require the length distribution to closely approximate the optimal length and be as concentrated as possible. Based on these conditions, we propose using the difficulty-level average as a well-founded reference for length optimization.
  Extensive experiments on both in-domain and out-of-domain benchmarks validate the superiority and effectiveness of DDPO. Compared to GRPO, DDPO reduces the average answer length by 12\% while improving accuracy by 1.85\% across multiple benchmarks, achieving a better trade-off between accuracy and length. The code is available at ~\url{https://github.com/Yinan-Xia/DDPO}.
\end{abstract}

\section{Introduction}
In recent years, many LRMs~\cite{guo2025deepseek,yang2025qwen3} have demonstrated remarkable capabilities, largely attributed to the progressively enhanced reasoning abilities. To address complex tasks, such as mathematical reasoning~\cite{shao2024deepseekmath}and code programming~\cite{jimenez2024swebench}, LRMs engage in Chain-of-Thought (CoT)~\cite{wei2022chain} to reflect on and revise their reasoning process before arriving at a final answer. This approach significantly enhances the comprehension and problem-solving skills of models. However, an excessive increase in reasoning may lead to the issue of \textit{overthinking}~\cite{sui2025stop,chen2024not,cuadron2025danger}, resulting in unnecessary computational resource consumption and, to some extent, a degradation in performance~\cite{cuadron2025danger}.

\begin{figure}[t]
  \begin{center}
\centerline{\includegraphics[width=0.9\columnwidth]{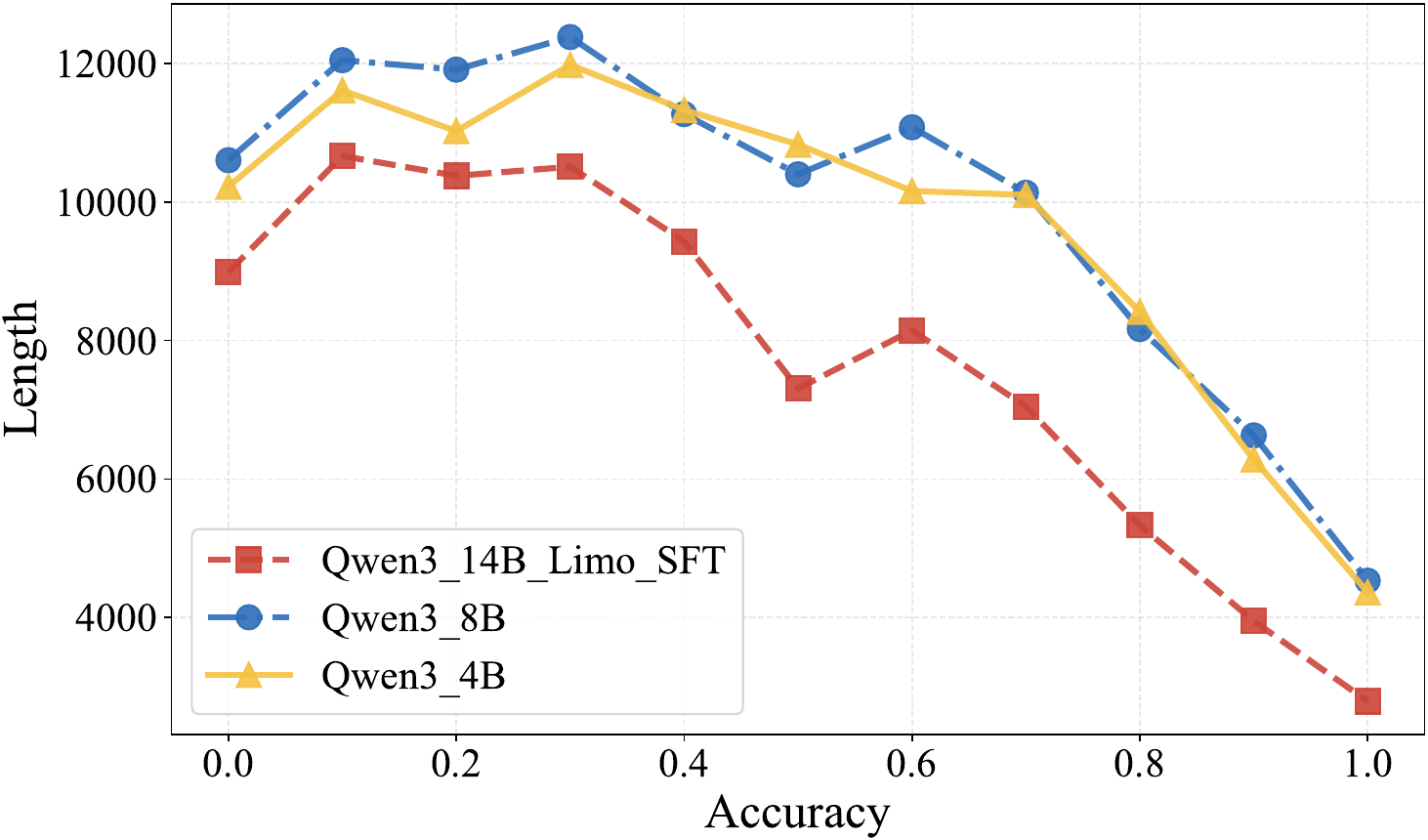}}
    \caption{
      The relationship between length and difficulties for different LRMs. The output length exhibits a non-monotonic relationship with problem difficulty: it first increases and then decreases as the difficulty decreases. For overly difficult problems, the LRMs generate shorter but incorrect answers, which we define as the 
      \textit{overconfidence} phenomenon.
    }
    \label{fig:line}
    \vspace{-2em}
  \end{center}
\end{figure}
\begin{figure*}[th]
\begin{center}
\centerline{\includegraphics[width=\textwidth]{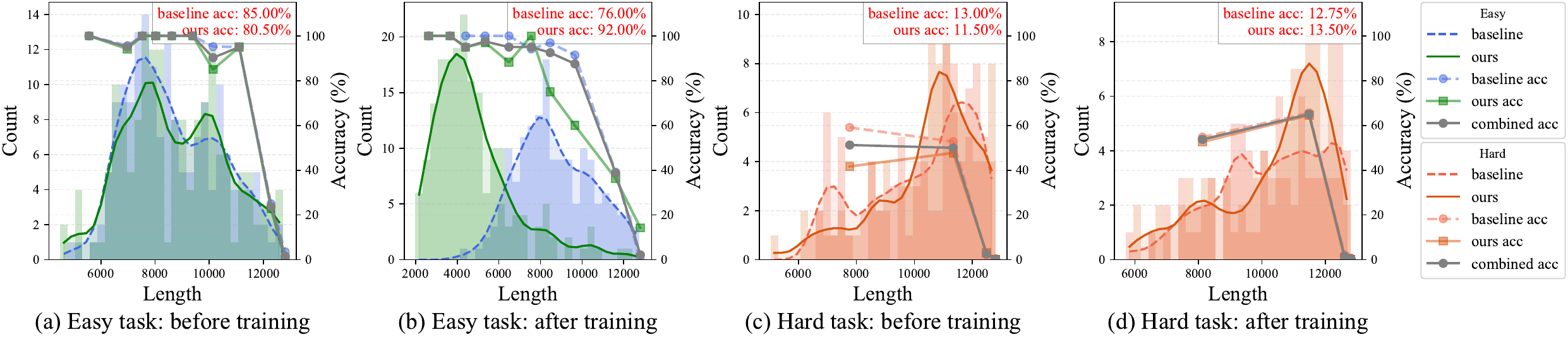}}
\caption{(a) Output length distribution and accuracy as a function of length on easy tasks before RL training.
(b) Length distribution and length-wise accuracy after GRPO ({\color[HTML]{4169e1} blue}) and DDPO ({\color[HTML]{377E23} green}) training. After training, DDPO yields a left-shifted and more concentrated length distribution compared to GRPO, with samples clustered around the optimal length (corresponding to the highest accuracy), resulting in higher accuracy.
(c) Output length distribution and accuracy as a function of length on hard tasks before RL training.
(d) Length distribution and length-wise accuracy after GRPO ({\color[HTML]{E07052} red}) and DDPO ({\color[HTML]{C75E2A}brick red}) training. After training, DDPO exhibits a clear rightward shift relative to GRPO, with a more concentrated distribution centered around the optimal length, leading to improved accuracy.} \label{fig:dist}
\vspace{-2em}
\end{center}
\end{figure*}
To mitigate this \textit{overthinking} phenomenon, several approaches~\cite{team2025kimi,luo2025o1,yu2025dapo} have been proposed to reduce the reasoning length, including controllable-length training to achieve a trade-off between resources and performance~\cite{aggarwal2025l1,lou2025adacot}, or implementing length penalties during the Reinforcement Learning (RL) process to enforce shorter outputs~\cite{yi2025shorterbetter,zhang2025grpo}. These approaches primarily focus on resource consumption through length reduction in an intuitive manner, without considering an evidence-based optimization approach based on the relationship between length and performance, thus failing to address the core challenge resulting from \textit{overthinking}: degraded performance.

To explore the relationship between reasoning length and performance, we present a line chart showing the reasoning length of the problems with varying accuracy levels across different model sizes in~\cref{fig:line}. It reveals that the relationship between length and accuracy is not positively correlated; rather, it exhibits an overall ``inverted U-shape.'' Interestingly, for extremely difficult problems, the reasoning length tends to be shorter; we define this phenomenon as \textit{overconfidence}, which reflects the model's overestimation of its own capabilities, leading to short but incorrect answers. This phenomenon restricts the model's ability to explore challenging problems, which we believe is a key factor contributing to its suboptimal accuracy. On the other hand, for simpler problems, the reasoning length tends to decrease as the accuracy increases. Many dynamic sampling approaches~\cite{yu2025dapo,bae2025online,qu2025can} assume that extremely difficult and overly simple problems contribute little to improving the model's efficiency and thus filter out these data points. However, these methods fail to fully leverage the potential of complex samples to enhance model performance, nor do they contribute to reducing the reasoning length for simpler problems.
In contrast, we propose categorizing the samples into difficult and simple problems and leveraging them to their fullest potential to enhance the model's capabilities across different aspects.

Additionally, we present the accuracy of the same problem at different reasoning lengths in the line chart of~\cref{fig:dist}, and define the length corresponding to the highest accuracy as the optimal length. Together with the reasoning length distribution in the histogram of~\cref{fig:dist}, we derive the following theoretical conclusions, with the formal proofs provided in~\cref{sec:optimal_length}: (1) The expected accuracy can be maximized as the expected length approaches the optimal length. (2) When the expected length equals the optimal length, a more concentrated length distribution results in a higher expected accuracy. Based on this, we aim to align the length as closely as possible with the optimal length. We approximate the average length of correct answers at the same difficulty level as the optimal length, and aim to make the model's expected output closely align with this optimal length, while ensuring a more concentrated distribution. As shown in~\cref{fig:dist}, compared to GRPO~\cite{shao2024deepseekmath}, our method produces a length distribution that is closer to the optimal length and more concentrated for both easy and hard problems, leading to higher accuracy. For easy problems in~\cref{fig:dist}(b), our method substantially reduces reasoning length without sacrificing performance, whereas for hard problems in~\cref{fig:dist}(d), it increases reasoning length to expand the exploration space and improve accuracy.
In summary, our contributions are as follows:

\begin{itemize}
    \item We propose Difficulty-Differentiated Policy Optimization (DDPO), a reinforcement learning algorithm that separately optimizes easy and hard problems to improve model efficiency and exploration capacity.
    \item We derive the key factors for leveraging the length distribution to maximize accuracy and develop a theoretically well-founded length optimization method to improve performance.
    \item Extensive experiments across in-domain and out-of-domain benchmarks on multiple baselines demonstrate the effectiveness and superiority of our approach.
\end{itemize}

\section{Related Work}
\subsection{Length Control Works in LLMs}
Recent works~\cite{chen2024not, arora2025training, su2025thinking} have addressed the overthinking phenomenon~\cite{zeng2025revisiting, muennighoff2025s1, luo2025deepscaler} in LRMs, which contributes to computational inefficiency and impaired performance~\cite{yi2025shorterbetter}. Studies~\cite{su2025between, wu2025more} show that increasing Chain-of-Thought (CoT) length does not continuously improve accuracy. Length-controlling strategies~\cite{team2025kimi, luo2025o1} fall into two categories: training-based and training-free. Training-based methods either use specialized data or apply length penalties or rewards in reinforcement learning (RL) frameworks.
For example, 
L1~\cite{aggarwal2025l1} optimizes for accuracy while adhering to user-specified length constraints, whereas ShorterBetter~\cite{yi2025shorterbetter} defines the optimal length as the shortest correct answer and penalizes other responses based on this ideal. AdaCtrl~\cite{huang2025adactrl} adjusts reasoning length based on a self-assessed evaluation of problem difficulty, focusing mainly on simpler questions. In contrast to the linear penalty approach used in these methods, AdaCoT~\cite{lou2025adacot} employs a binary penalty to decide whether to invoke Chain-of-Thought (CoT) reasoning. Training-free methods, on the other hand, control output length by prompting the model to generate concise contexts or by sampling and pruning responses to select those that are both succinct and informative~\cite{muennighoff2025s1,chen2024unlocking,liu2025adaptivestep,hou2025thinkprune}. Despite these efforts to consciously control output length, most research has primarily focused on reducing length, with limited attention given to the potential benefits of increasing it. 
\subsection{Difficulty-Adaptive Works in LLMs}
The difficulty of the data plays a crucial role in training LLMs, as appropriately challenging training data leads to more efficient model development and better generalization.
A notable example of training with variable difficulty data is curriculum learning~\cite{narvekar2018learning,narvekar2020curriculum}, where the training data is organized from simple to complex.
Shorter but not Worse~\cite{bounhar2025shorter} initially capitalized on the concise nature of easier problems as an implicit length regularizer and later incorporated more challenging issues to enhance the model's capabilities. QwenLong L1~\cite{wan2025qwenlong} employed a difficulty-aware retrospective sampling strategy to incentivize policy exploration, showcasing its effectiveness in long-context reasoning tasks. Moreover, difficulty-adaptive algorithms are specifically designed to handle data with varying difficulties, resulting in enhanced model performance. For example, TP-GRPO~\cite{he2025good} adjusts the constraint strength of the Process Reward Model (PRM) based on problem difficulty, applying more lenient intermediate constraints for complex problems to encourage greater exploration. 
GRPO-LEAD~\cite{zhang2025grpo} introduces a difficulty-aware advantage mechanism: for difficult problems, correct responses receive larger updates, while incorrect responses to easier problems are penalized more.
Additionally, other studies~\cite{shen2025dast, zhang2025continue} incorporate difficulty-based length penalties in the reward function, applying stronger penalties to simpler problems and lighter penalties to more complex ones.
The practical applications of the aforementioned works highlight the feasibility of optimizing models by differentiating training data based on difficulty. 
However, existing methods typically divide problem difficulty based on heuristic or empirical criteria. In contrast, our approach defines problem difficulty according to the observed \textit{overconfidence} phenomenon.

\begin{figure*}[th]
\begin{center}
\centerline{\includegraphics[width=0.9\textwidth]{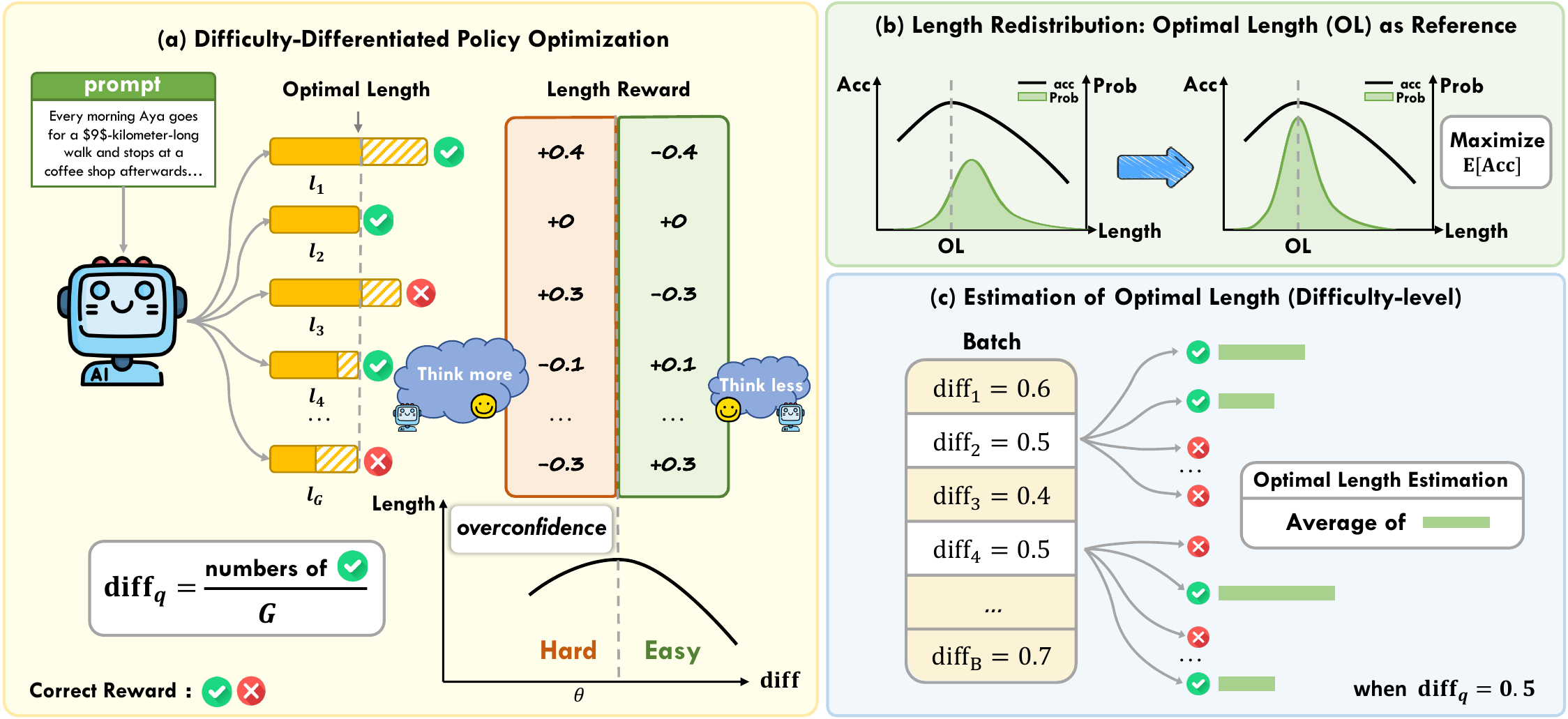}}
\caption{(a) We define the rollout accuracy as an indicator of the difficulty for each query. By analyzing the relationship between difficulty and output length (shown in the lower-right panel), we observe a inverted U-shaped pattern, which we define as ``\textit{overconfidence}''. Based on this observation, we categorize queries into hard and easy ones and propose Difficulty-Differentiated Policy Optimization (DDPO), which encourages greater exploration for difficult queries while shrinking output length for easy queries to avoid redundancy. (b) When applying the length optimization strategy, we use the optimal length as a reference and guide the final length distribution to be closer to this value and more concentrated, thereby increasing the expected accuracy. (c) We estimate the optimal length as the average length of correct answers for queries of the same difficulty within the batch.
} \label{fig:method}
\vspace{-2em}
\end{center}
\end{figure*}

\section{Method}
In this work, we aim to explore the relationship between reasoning length and model performance, with the goal of effectively optimizing the model's output length to simultaneously enhance its performance.
In~\cref{sec:grpo}, we provide a brief introduction of our baseline algorithm, GRPO.
In~\cref{sec:DDO}, we identify the phenomenon of \textit{overconfidence} in LRMs and propose a method for difficulty-differentiated optimization. In~\cref{sec:optimal_length}, we establish the conditions for the length distribution that maximizes the expected accuracy and further refine our approach for length optimization.
\subsection{Group Relative Policy Optimization} \label{sec:grpo}
Group Relative Policy Optimization (GRPO)~\cite{shao2024deepseekmath} replaces the additional value function approximation in Proximal Policy Optimization (PPO)~\cite{schulman2017proximal} with the average reward computed from multiple sampled outputs of the same question. Specifically, for each question $q$, GRPO samples a group of outputs $\{o_1,o_2,...,o_G\}$ from the old policy $\pi_{\theta_{old}}$ to estimate the expected reward, and then optimizes the policy model $\pi_{\theta}$ by maximizing the following objective:
\begin{equation}
\begin{aligned}
&\mathcal{J}_{GRPO}(\theta) =\mathbb{E}[q\sim P(Q),\{o_i\}_{i=1}^G\sim\pi_{\theta_{old}}(O|q)] \\
 & \frac{1}{G}\sum_{i=1}^G\left(\min\left(w_{i}A_i,\operatorname{clip}\left(w_i,1-\varepsilon,1+\varepsilon\right)A_i\right)-\beta\text{KL}\right), 
\end{aligned}
\end{equation}
\begin{align}
    w_i     &= \frac{\pi_\theta(o_i|q)}{\pi_{\theta_{old}}(o_i|q)}, \label{eq:wi} \\
    \text{KL} &= w_i - \log w_i - 1, \label{eq:kl}
\end{align}
where $\varepsilon$ and $\beta$ are hyper-parameters, and $A_i$ is the advantage computed from the rewards $\{r_1,r_2,...,r_G\}$ corresponding to the output group:
\begin{align}
A_i=\frac{r_i-\text{mean}(\{r_1,r_2,\cdots,r_G\})}{\mathrm{std}(\{r_1,r_2,\cdots,r_G\})},
\end{align}
Due to the group algorithm in GRPO, we can calculate the accuracy of each query during the current training process, which we use as a measure of the query's difficulty:
\begin{align} \label{eq:diff}
    \text{diff}_q=\frac{\sum_{i=1}^Gr_i}{G},
\end{align}
where $r_i$ is a binary value (0 or 1) used to measure whether the output $o_i$ is correct. 

\subsection{Difficulty-Differentiated Policy Optimization}\label{sec:DDO}
\textbf{The Overconfidence Phenomenon in LRMs.} 
To explore the relationship between reasoning length and model performance, we visualize the output lengths for problems of varying difficulties in LRMs in~\cref{fig:line}. Specifically, we use different LRMs to generate 10 rollouts for each query in validation set. By calculating the difficulty of each query using~\cref{eq:diff}, we 
divide the queries into 6 intervals based on difficulty, ranging from 0 to 1 in steps of 0.2. Then we plot the line chart of output lengths for each difficulty level in~\cref{fig:line}. 
The line chart shows an initial increase in length, followed by a decrease as the difficulty decreases, which contradicts the intuition that answer length should increase with difficulty. We attribute this counterintuitive phenomenon—where samples with an accuracy of 0 exhibit significantly shorter lengths than those with an accuracy of 0.2—to the \textit{overconfidence} of LRMs on overly difficult questions beyond the model's capability, resulting in shorter but incorrect answers, which likely contributes to suboptimal performance.
For other problems, where length decreases as accuracy increases, we believe these correspond to relatively simple tasks within the model's capabilities. In such cases, there is still potential to further reduce the output length.

\textbf{Difficulty-Differentiated Policy Optimization.} Based on the \textit{overconfidence} phenomenon, we categorize problems into difficult and easy subsets by using a threshold $\theta$ for differentiation. Problems with accuracy lower than $\theta$ ($\theta\in[0,1]$) are classified as difficult, while those with accuracy greater than or equal to $\theta$ are classified as easy. As shown in~\cref{fig:method}, for these two categories, we implemented distinct optimization strategies as follows.

For difficult problems, we aim to increase the output length to expand the model's exploration space, thereby improving accuracy. To prevent excessive exploration that may lead to truncated answers, we set an upper bound $L_\text{high}$ and avoid encouraging answers beyond this limit. Formally, for a question $q$ with difficulty $\text{diff}_q<\theta$, its reward $r_i$ is defined as:
\begin{align} \label{eq:hard}
r_i = 
\begin{cases} 
r_i^\text{correct} + \alpha_{\text{hard}} \cdot z_i, & \text{if } L_i \leq L_\text{high} \\
r_i^{correct}, & \text{if } L_i > L_\text{high}
\end{cases}
\end{align}
where $r_i^\text{correct}$ is the correct reward of the output $o_i$, and $z_i=L_i/L_\text{max}$, with $L_\text{max}$ being the maximum length used for normalization. Notably, the coefficient is defined as $\alpha_\text{hard}=\theta-\text{diff}_q$, indicating that the degree of exploration increases with problem difficulty.

For a simple problem $q$ with $\text{diff}_q\geq\theta$, we tend to shorten the answer length without compromising accuracy:
\begin{align} \label{eq:easy}
r_i = 
\begin{cases} 
r_i^\text{correct} - \alpha_{\text{easy}} \cdot z_i, & \text{if } L_i \geq L_\text{low} \\
r_i^\text{correct}, & \text{if } L_i < L_\text{low}
\end{cases}
\end{align}
where $L_\text{low}$ is the lower bound to prevent over-penalty. And $\alpha_\text{easy}=\text{diff}_q$, implying that the simpler the problem, the stronger the length constraint applied.

\subsection{Towards Optimal Performance}\label{sec:optimal_length}
In addition to difficulty-differentiated length optimization, the direction and extent of optimization are also crucial. This insight motivates us to refine the length optimization algorithm by focusing on how length influences model performance, enabling more precise and rational control of length for optimal outcomes.

\textbf{Existence of the Optimal Length.} 
We present the relationship between length and accuracy for both simple and difficult tasks in~\cref{fig:dist} (the line chart). 
Specifically, we sampled 200 answers per question using fine-tuned Qwen3-14B, dividing the entire length range into $n$ bins to ensure an even distribution of samples, and then calculated the accuracy for each bin. As shown in the line chart in~\cref{fig:dist}, the relationship between answer length and accuracy follows an approximately convex curve, where performance initially improves and then deteriorates as length increases, peaking at a certain point. Based on this observation, we propose the following hypothesis:
\begin{assumption}\label{assumption:optimal}
\textit{For a fixed model, the length of answers is directly related to their accuracy, with each question having an optimal answer length that maximizes accuracy.}
\end{assumption}
\textbf{Maximizing Expected Accuracy Theoretically.}
We define the reasoning length $l$ with a probability density function $p(l)$, and let $\mu$ and $\sigma^2$ denote its mean and variance. Based on~\cref{assumption:optimal}, we define $f(l)$ as the accuracy corresponding to reasoning length $l$, assuming that $f(l)$ is a concave function with a unique maximum at $l^*$, i.e., $f(l^*) = \max\left(f(l)\right)$.

\begin{lemma} \label{lemma:mu}
The expected accuracy $E(f(l))$ can be maximized $f(l^*)$ when $\mu = l^*$.
\end{lemma}
\textit{Proof.} For a concave function $f$, the Jensen inequality~\cite{mcshane1937jensen} holds:
\[
f(E[l]) \geq E[f(l)],
\]
with equality if and only if $l$ is deterministic, i.e., $l = \mu$. By the definition of concavity, if $\mu = l^*$, we have $f(l^*) \geq E[f(l)]$, and $E[f(l)]$ reaches the theoretical maximum value $f(l^*)$. If $\mu \neq l^*$, combining Jensen's inequality, we obtain:
\[
f(l^*) > f(\mu), \quad \text{so} \quad E[f(l)] \leq f(\mu) < f(l^*).
\]
where $E[f(l)]$ cannot reach the optimal accuracy. 
\begin{lemma}\label{lemma:var}
For $\mu = l^*$, the expected accuracy $E[f(l)]$ is maximized when the variance $\sigma^2$ is minimized.
\end{lemma}
\textit{Proof:} 
Larger variance widens the distribution of $l$, decreasing the likelihood of it being close to the optimal $l^*$, so minimizing variance improves accuracy.

\begin{table*}[t]
\caption{(Upper) Main results compared with other length optimization methods.The {\color[HTML]{C00000}\textbf{red}} and {\color[HTML]{4472C4}\textbf{blue}} represent the best and second-best result respectively. Our method can best balance the accuracy and the length. (Lower) Experiments on other models with different sizes. Our method can enhance the performance and shorten the length across varying models.}\label{tab:compare}
\vspace{-0.8em}
\renewcommand\arraystretch{1.15}
\fontsize{6}{7}\selectfont\setlength{\tabcolsep}{1.15mm}
\begin{center}
\begin{small}
\begin{sc}
\begin{tabular}{lcccccccccccc|cc}

\toprule
\multirow{2}{*}{Methods/Models}
 & \multicolumn{2}{c}{\textbf{OLYMPIAD}}                         & \multicolumn{2}{c}{\textbf{MATH}}                                    & \multicolumn{2}{c}{\textbf{AMC}}                                     & \multicolumn{2}{c}{\textbf{AIME2025}}                                & \multicolumn{2}{c}{\textbf{AIME2024}}                               & \multicolumn{2}{c}{\textbf{GPQA-D}}                                & \multicolumn{2}{c}{\textbf{AVG.}}   \\
 & acc. & len. & acc. & len.& acc. & len.& acc. & len.& acc. & len.& acc. & len.& acc. & len. \\
\midrule
\text{GRPO} 
 & $55.39$ & $5885$ & \color[HTML]{4472C4}$\textbf{91.86}$ & $3026$ & $89.84$ & $4080$ & $42.71$ & $9569$ & $56.04$ & $8927$ & $57.11$ & $5728$ &  $65.49$ & $6324$ \\
\text{L1} & \color[HTML]{C00000} $\textbf{57.67}$ & $6262$ & \color[HTML]{C00000} $\textbf{92.41}$ & $4269$ & \color[HTML]{4472C4} $\textbf{90.31}$ & $4957$ & $42.92$ & \color[HTML]{4472C4} $\textbf{9023}$ & \color[HTML]{4472C4} $\textbf{59.38}$ & $8090$& \color[HTML]{4472C4}$\textbf{57.26}$& $5543$ & \color[HTML]{4472C4} $\textbf{66.66}$ & $6357$ \\
\text{Shorter Better} 
 & $54.23$ & \color[HTML]{4472C4}$\textbf{4920}$ & $90.75$ & \color[HTML]{C00000}$\textbf{1920}$ & $88.44$ & \color[HTML]{C00000}$\textbf{3600}$ & $42.71$ & $9519$ & $58.96$ & \color[HTML]{C00000}$\textbf{7246}$ &$56.88$&\color[HTML]{C00000}$\textbf{4416}$& $65.33 $ & \color[HTML]{C00000}$\textbf{5270}$\\

\text{A-DLP} 
 & $54.03$ & \color[HTML]{C00000}$\textbf{4419}$ & $90.28$ & \color[HTML]{4472C4}$\textbf{2052}$ & $88.13$ &$3998$ &$43.13$ & $9248$ & $58.75$ & $8519$&$56.41$& \color[HTML]{4472C4}$\textbf{4986}$& $65.12 $ & \color[HTML]{4472C4}$\textbf{5536}$ \\
\text{GRPO-LEAD} 
 & $53.64 $ & $6173$ & $91.00$ & $3032 $ & $88.44 $ & $4823$ & \color[HTML]{4472C4}$\textbf{44.17 }$ & $9432 $ & $56.88 $ & $9010 $&$55.01 $& $7147 $& $64.86 $ & $6603 $ \\
\text{OURS} 
 & \color[HTML]{4472C4}$\textbf{55.70 }$ & $5299 $ & $91.84$ & $2321 $ & \color[HTML]{C00000}$\textbf{91.09 }$ & \color[HTML]{4472C4}$\textbf{3622}$ & \color[HTML]{C00000}$\textbf{47.08 }$ & \color[HTML]{C00000}$\textbf{8359 }$ & \color[HTML]{C00000}$\textbf{60.63 }$ & \color[HTML]{4472C4}$\textbf{8074 }$ & \color[HTML]{C00000}$\textbf{57.68 }$ & $5683 $ & \color[HTML]{C00000}$\textbf{67.34} $ & $5560 $ \\
 
 \midrule
 Qwen3-8b+grpo &$57.92$&$7372$&$\textbf{92.96}$&$4212$&$90.31 $&$6013$&$49.38$&$10281$&$63.96$&$9242$&$\textbf{56.76}$&$7379$&$68.55$&$7416$ \\
\rowcolor{gray!12}
Qwen3-8b+ours
 & $\textbf{58.95} $ &$\textbf{5501}$&$92.95$&$\textbf{2693}$&$\textbf{90.47}$&$\textbf{4445}$&$\textbf{50.83}$&$\textbf{8881}$&$\textbf{65.63}$&$\textbf{7941}$&$56.35$&$\textbf{6363}$&$\textbf{69.20}$&$\textbf{5970}$ \\
Qwen3-4b+grpo
 & $56.41$&$7347$ &$92.43$ &$3992$ &$88.13$ &$6031$ &$\textbf{49.38}$ &$10139$ &$61.87$ &$9288$ &$51.02$ &$7423$ &$66.54$ &$7370$  \\
 \rowcolor{gray!12}
Qwen3-4b+ours
&$\textbf{57.27}$ & $\textbf{5721}$&$\textbf{92.75}$ &$\textbf{2938}$ &$\textbf{91.56}$ &$\textbf{4074}$ &$48.96$ &$\textbf{8239}$ &$\textbf{65.83}$ &$\textbf{7972}$ &$\textbf{51.84}$ &$\textbf{6252}$ &$\textbf{68.03}$ &$\textbf{5866}$  \\
\bottomrule
\vspace{-2.5em}
\end{tabular}
\end{sc}
\end{small}
\end{center}
\end{table*}

Overall, according to~\cref{lemma:mu} and~\cref{lemma:var}, we have established our goals: i) making the expected value of the length distribution as close as possible to the optimal length. ii) ensuring that the length distribution is as concentrated as possible.

\textbf{Optimal Length as the Reference.} Building on the analysis in~\cref{sec:DDO}, we aim to align the length as closely as possible with the optimal length. For a given question $q$, we approximate the mean length of the correct answers for problems in the batch with difficulty $\text{diff}_q$ as the optimal length for $q$. Based on this, we modify the normalization in~\cref{eq:hard} and~\cref{eq:easy} as follows:
\begin{align} \label{eq:optimal}
    z_i = \frac{L_i - \mu_{\text{diff}_q}}{L_\text{max}},
\end{align}
\vspace{-2em}
\begin{align}
    \mu_{\text{diff}_q} = \frac{1}{C} \sum_{k=1}^{B}\sum_{i=1}^{G} \mathbb{I}(r_i^{correct} = 1 | \text{diff}_{q_k} = \text{diff}_q) \cdot l_i,
\end{align}
where $B$ is the batch size, $G$ is the group size, $l_i$ is the output length of $o_i$, and $C=\sum_{k=1}^{B}\sum_{i=1}^{G} \mathbb{I}(r_i^{correct} = 1| \text{diff}_{q_k} = \text{diff}_q)$. 
For simple tasks, the rollout length is penalized if it exceeds the optimal length and rewarded if it is shorter, leading to a reduction in overall length. For difficult tasks, the length is penalized if it is shorter than the optimal length and rewarded if it exceeds the optimal length, promoting longer outputs and encouraging exploration. At the same time, normalizing by the optimal length brings the lengths closer to it, making the length distribution more concentrated.
Notably, problems with $\text{diff}_q=0$ have no corresponding correct answers and represent the most difficult problems. Therefore, we aim to maximize exploration for these problems and define $z_i = L_i/L_{\text{max}}$. More detailed analysis is provided in Appendix~\ref{app:acc=0}.

Some recent works~\cite{yi2025shorterbetter} use the average length of the correct answers within the same query group as the optimal length. However, normalizing by the query-level average restricts comparisons to answers across different queries, limiting the informative cross-query exploration and ultimately leading to uniform rollout lengths within each query.
Additionally, using the length of correct answers across the entire batch as the optimal length is a straightforward and intuitive approach. However, for simple tasks, the average length of the batch is often longer than the rollout lengths, leading to an over-reward in~\cref{eq:easy} and causing excessively long responses. 
For difficult problems, the batch average length may be smaller than rollout lengths, which limits the exploration space and hinders optimal performance. 
We argue that the exploration space for answering questions depends on the problem's difficulty, and thus, we normalize the length using the difficulty-related average as the optimal length. Additionally, varying penalty and reward signals across queries with the same difficulty can improve learning efficiency. Additional comparison experimental results are provided in~\cref{tab:mu}.

\section{Experiments}
\subsection{Experiment Setup}
\textbf{Training Details.}
We perform RL training on the Qwen3-14B~\cite{yang2025qwen3} model, which is initially fine-tuned with LIMO-v2~\cite{ye2025limoreasoning} using SFT. For RL training, we employ the DeepScaleR-Preview-Dataset~\cite{deepscaler2025}, with a maximum sequence length of 12,800, a batch size of 128, a temperature of 1.0, a top-p value of 1, 10 rollouts, and a learning rate of $1e-6$.

\textbf{Evaluation Details.}
We conduct extensive experiments primarily on mathematical tasks, encompassing OlympiadBench~\cite{he2024olympiadbench}, MATH500~\cite{lightman2023let}, AMC\footnote{\url{https://huggingface.co/datasets/AI-MO/aimo-validation-amc}}, AIME2025\footnote{\url{https://huggingface.co/datasets/opencompass/AIME2025}}, AIME2024\footnote{\url{https://huggingface.co/datasets/HuggingFaceH4/aime\_2024}}, together with the general benchmark GPQA Diamond~\cite{rein2024gpqa}.
To further assess the out-of-domain generalization, we conduct additional experiments on MathQA~\cite{amini2019mathqa}, MMLU-Pro~\cite{wang2024mmlu}, BBEH~\cite{kazemi2025big}. As for the evaluation, we set the number of rollouts to 16, the top-p to 0.95, and the temperature to 0.6.

\textbf{Base Models.} In addition to Qwen3-14B, we validated the effectiveness of our method on Qwen3-8B and Qwen3-4B.

\textbf{Baselines.}
To validate the broad applicability of our method. We applied our method to GRPO~\cite{shao2024deepseekmath}, DAPO~\cite{yu2025dapo}, and DrGRPO~\cite{liu2025understanding}.

\textbf{Compared Methods.} We compare our approach with recent length optimization methods, including L1~\cite{aggarwal2025l1}, ShorterBetter~\cite{yi2025shorterbetter}, A-DLP~\cite{su2025thinking}, and GRPO-LEAD~\cite{zhang2025grpo}.
\subsection{Main Results}
\textbf{Comparing to Other Methods.} 
As shown in~\cref{tab:compare} (Upper), we compared the performance of our method ($\theta=0.4$) with recent approaches across multiple benchmarks, in terms of both accuracy and reasoning length. Our approach outperformed others, achieving the highest average accuracy, which improved by 1.85\% compared to L1, while also reducing the output length by 12\%.
Methods like ShorterBetter and A-DLP generate very short responses but fail to attain high accuracy. In contrast, our method strikes a balance, achieving higher accuracy and relatively shorter response lengths on average, and delivers the best performance on challenging benchmarks like AIME2024 and AIME2025. As for the easy benchmark like AMC, our method reaches the best accuracy and the second-best length performance.
Furthermore, our method outperforms existing approaches in accuracy on out-of-domain general datasets, such as GPQA Diamond, while also generating shorter outputs, highlighting the superior generalization ability of DDPO.

\textbf{Scaling to Models of Different Sizes.}
In addition to applying our method to the Qwen3-14B model after SFT, we also applied DDPO ($\theta=0.2$) to models with different sizes, namely Qwen3-8B and Qwen3-4B, which had not undergone additional SFT. Experimental results in~\cref{tab:compare} (Lower) show that, without any additional SFT, directly applying DDPO to Qwen3 models of different sizes improves model performance more significantly. In particular, applying DDPO to Qwen3-8B results in notable performance improvements in in-domain benchmarks, along with a clear reduction in output length. On the out-of-domain GQPA Diamond dataset, the length decreased significantly without sacrificing much performance. Overall, the average accuracy increased by 0.65\%, while the length decreased by 19.5\% on Qwen3-8B. A similar conclusion can be drawn for the Qwen3-4B models, which demonstrates that controlling length to enhance accuracy is a viable approach.

\begin{table}[t]
\caption{The effectiveness of our method across different baselines. The {\textcolor{gray}{gray}} background indicate the performance of our method ($\theta=0.2$) adding on varying baselines; \textbf{Bolded} values denote the better performance between baseline and our method.
}\label{tab:baseline}
\vspace{-0.8em}
\renewcommand\arraystretch{1.1}
\fontsize{7}{8}\selectfont\setlength{\tabcolsep}{7mm}
\begin{center}
\begin{small}
\begin{sc}
\begin{tabular}{lcc}

\toprule
\multirow{2}{*}{Baselines}
& \multicolumn{2}{c}{\textbf{AVG.}}   \\
 & acc. & len. \\
\midrule
DAPO &$65.87$&$\textbf{5236}$ \\

\rowcolor{gray!10}
DAPO+ours
 &$\textbf{66.61} $&$5490$ \\
DrGRPO
 & $65.22 $ & $6528$ \\
\rowcolor{gray!10}
 DrGRPO+ours
 & $\textbf{65.99} $ & $\textbf{5824}$ \\
 \text{GRPO} 
 &  $65.49$ & $6324$ \\
\rowcolor{gray!10}
 \text{GRPO+ours} 
 & $\textbf{66.79} $ & $\textbf{5494}$ \\
\bottomrule
\vspace{-10mm}
\end{tabular}
\end{sc}
\end{small}
\end{center}
\end{table}

\textbf{Applying to Other Baselines.}
To validate the broad applicability of our method, we applied our method not only to GRPO but also to its variants DAPO~\cite{yu2025dapo} and DrGRPO~\cite{liu2025understanding}, based on the fine-tuned Qwen3-14B model.~\cref{tab:baseline} presents the average accuracy and length on the six key benchmarks. The results demonstrate that our method improves the accuracy of various baselines while reducing the answer length in most cases. Specifically, it reduces the length of GRPO by 13\% and the length of DrGRPO by 11\%, which indicates the excellent scalability and broad applicability of DDPO. Full results and additional analysis are presented in Appendix~\ref{app:baselines}.

\subsection{Ablation Studies}
We systematically evaluated the importance of each mechanism in~\cref{tab:ablation}, based on the fine-tuned Qwen3-14B model, with $\theta=0.4$. The mechanisms tested include: 
i) Initial normalization ($L_i/L_\text{max}$) for length adjustment. 
ii) Difficulty-related coefficients ($\alpha$), which enabled more fine-grained length optimization. This not only improved model performance but also reduced output length.
iii) Normalization with optimal length $\mu_{\text{diff}_q}$ as the reference. By using the average length at each difficulty level as a reference for normalization, the model was able to better target the optimal output length, further improving accuracy while reducing unnecessary length.
iv) Lower Bar (LB), which constrained overly short answers. While slightly increasing the average length, it contributed to enhanced accuracy.
v) Upper Bar (UB), which mitigated the issue of excessively long outputs being truncated, thus maintaining both accuracy and length constraints. Full results are provided in  Appendix \ref{app:ablation}. 

\begin{table}[t]
\caption{Ablation Study. 
We validated the importance of each mechanism, showing that each improves accuracy and length control, with the optimal combination yielding significant performance gains and more concise output.
}\label{tab:ablation}
\vspace{-0.8em}
\renewcommand\arraystretch{1.25}
\fontsize{7}{8}\selectfont\setlength{\tabcolsep}{2.5mm}
\begin{center}
\begin{small}
\begin{sc}
\begin{tabular}{lcc}

\toprule
\multirow{2}{*}{Components}
& \multicolumn{2}{c}{\textbf{AVG.}}   \\
 & acc. & len. \\
\midrule
$L_i/L_\text{max}$ &$64.51$ &$5881$ \\
$\alpha*L_i/L_\text{max}$ &$65.02$ &$5692 $ \\
$\alpha*(L_i-\mu_{\text{diff}_q})/L_\text{max}$ &$65.47$ & $5968 $ \\
$\alpha*(L_i-\mu_{\text{diff}_q})/L_\text{max}+LB$ &$65.95$ &$6185 $\\
$\alpha*(L_i-\mu_{\text{diff}_q})/L_\text{max}+LB+UB$ &$\textbf{67.34}$ &$\textbf{5560}$ \\
\bottomrule
\vspace{-10mm}
\end{tabular}
\end{sc}
\end{small}
\end{center}
\end{table}

\begin{table*}[t]
\caption{Analysis of optimal length approximation ($\theta$=0.2). 
The results demonstrate that approximating the optimal length by difficulty-level averages can achieve optimal average performance. 
The \textbf{bold} represents the best result. 
}\label{tab:mu}
\vspace{-0.8em}
\renewcommand\arraystretch{1.15}
\fontsize{6}{7}\selectfont\setlength{\tabcolsep}{1.5mm}
\begin{center}
\begin{small}
\begin{sc}
\begin{tabular}{lcccccccccccc|cc}

\toprule
\multirow{2}{*}{AVG-TYPE}
 & \multicolumn{2}{c}{\textbf{OLYMPIAD}}                         & \multicolumn{2}{c}{\textbf{MATH}}                                    & \multicolumn{2}{c}{\textbf{AMC}}                                     & \multicolumn{2}{c}{\textbf{AIME2025}}                                & \multicolumn{2}{c}{\textbf{AIME2024}}                               & \multicolumn{2}{c}{\textbf{GPQA-D}}                                & \multicolumn{2}{c}{\textbf{AVG.}}   \\
  & acc.& len. & acc. & len.& acc. & len.& acc. & len.& acc. & len.& acc. & len.& acc. & len. \\
\midrule
$\mu_q$ &$55.46$&$\textbf{4954}$&$91.80 $&$2394 $&$89.38 $&$\textbf{3846 }$&$43.33 $&$\textbf{8104 }$&$57.71 $&$8315 $&$56.79 $&$5931 $&$65.74 $&$5591 $ \\

$\mu_\text{batch}$
 & $\textbf{55.60} $ &$6273 $&$91.55 $&$4110 $&$89.06 $&$5113 $&$\textbf{44.79 }$&$9486 $&$59.79 $&$8549 $&$57.39 $&$6462 $&$66.36 $&$6665 $ \\

$\mu_{\text{diff}_q}$
 & $55.42$ & $5027$ & $\textbf{92.05}$ & $\textbf{2165}$ & $\textbf{89.84}$ & $4056$ & $44.17$ & $9204$ & $\textbf{61.46}$ & $\textbf{7119}$ & $\textbf{57.80}$ & $5390$ & $\textbf{66.79} $ & $\textbf{5494}$ \\
\bottomrule
\end{tabular}
\end{sc}
\end{small}
\end{center}
\end{table*}

\subsection{Additional Analysis}
\textbf{Analysis of the Estimation of the Optimal Length.}
In~\cref{eq:optimal}, we approximate the optimal length as the average length of correct answers for queries within the batch that share the same difficulty as the current query, denoted as $\mu_{\text{diff}_q}$, which serves as the reference for normalization. To validate the rationale of this approximation, we compared $\mu_{\text{diff}_q}$ with other variants: i) the average length of correct answers in the current query (called $\mu_q$); ii) the average length of correct answers across the entire batch (called $\mu_\text{batch}$).
As shown in~\cref{tab:mu}, the query-level average $\mu_q$ performs worse on challenging benchmarks such as AIME2024 and AIME2025 compared to the difficulty-level average $\mu_{\text{diff}q}$, suggesting its insufficient exploration for harder problems. Additionally, the length of the $\mu_\text{batch}$ method on simpler benchmarks like Math and AMC is significantly larger than both $\mu_q$ and $\mu_{\text{diff}q}$, as excessive length rewards lead to redundant answers. 
In contrast, normalizing by difficulty-level average $\mu_{\text{diff}q}$ results in competitive performance in both accuracy and length. As analyzed in~\cref{sec:optimal_length}, the exploration space for answering questions varies with the difficulty of
the problem, and the experimental results validate the rationale for using $\mu_{\text{diff}_q}$ as the optimal reference, ensuring a more tailored and efficient optimization.

\textbf{The Shift of Length Distribution.}
To validate the effectiveness of our theory and approach presented in \cref{sec:optimal_length}, we visualized the distribution of answer lengths for both our DDPO and the baseline GRPO, before and after training, across easy and hard tasks in~\cref{fig:dist}. 
Additionally, we illustrate the accuracy corresponding to these lengths (the line plots in~\cref{fig:dist}), where the optimal length is defined as the length of the highest accuracy. In this case, we use AMC and AIME2025 to represent the easy and hard tasks in~\cref{fig:dist} (a)(b) and~\cref{fig:dist} (c)(d), respectively.

For the easy task shown in \cref{fig:dist} (a)(b), the optimal length is shorter, whereas for the hard task in \cref{fig:dist} (c)(d), the optimal length is typically longer. Compared to GRPO, our method shifts the length distribution leftward for easy tasks, making it more concentrated and closer to the optimal length. For hard tasks, it shifts the length distribution rightward, concentrating it around the optimal length. These distributional shifts for both task types result in significant accuracy improvements over GRPO, with accuracy increasing by 11.5\% for easy tasks and by 11.5\% to 13.5\% for hard tasks. The results confirm that our method effectively satisfies \cref{lemma:mu} and \cref{lemma:var}, aligning the expected length with the optimal length and producing a more compact distribution, thereby enhancing model accuracy.

\begin{figure}[t]
\begin{center}
\centerline{\includegraphics[width=0.5\textwidth]{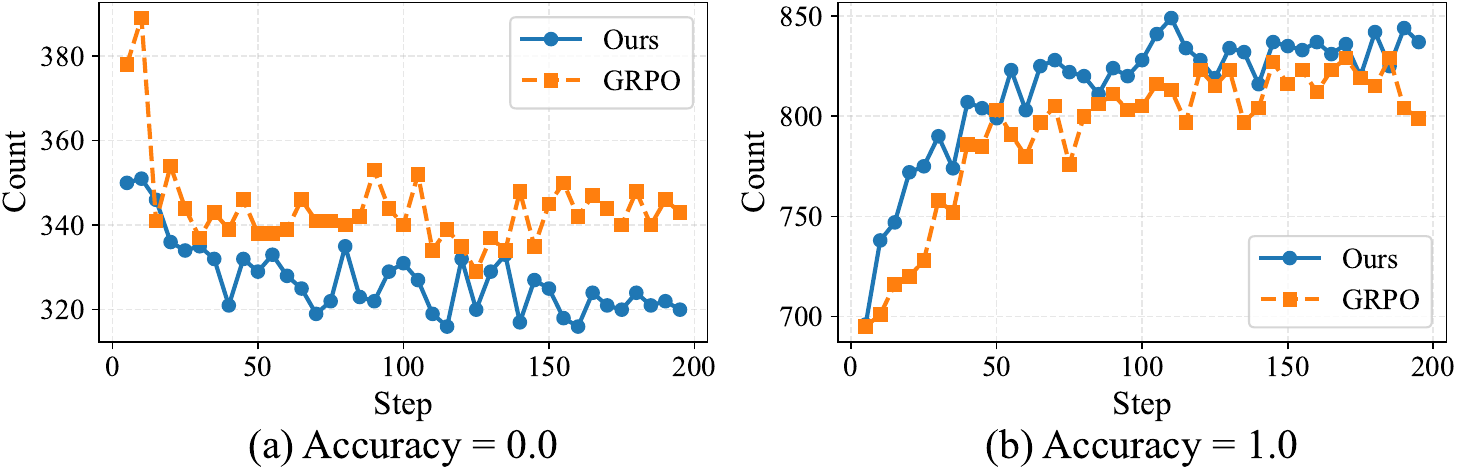}}
\caption{The change in the number of samples with an accuracy of 0 and 1 during the training process of GRPO and DDPO. Throughout DDPO training, the number of samples with an accuracy of 0 is consistently lower than that in GRPO, while the number of samples with an accuracy of 1 consistently exceeds that in GRPO.} \label{fig:acc}
\vspace{-2em}
\end{center}
\end{figure}

\textbf{Effective Utilization of Extreme Samples}.
To demonstrate that our method effectively addresses extreme samples (with accuracy of 0 and 1) and thereby increases the model accuracy, we visualized the trends of such samples across the entire validation set during training. As shown in \cref{fig:acc}, the number of samples with accuracy of 0 steadily decreases compared to GRPO, while the number of samples with an accuracy of 1 remains consistently higher. This indicates that by focusing on challenging samples, our method has expanded the model’s exploration space, leading to improved accuracy. Moreover, the reduction in length for simpler samples does not degrade the quality of the model’s responses. It confirms that our method effectively leverages extreme samples to enhance the model's performance.

\textbf{A More Concentrated Length Distribution}. In~\cref{sec:optimal_length}, we derive the conditions for maximizing expected accuracy, namely that the length distribution should be close to the optimal length and highly concentrated. Guided by these conditions, we implement a normalization strategy using the difficulty-level average length as the reference in~\cref{eq:optimal}. To validate the effectiveness of the approach, we visualize the box plots of output lengths for problems of varying accuracy in~\cref{fig:box}. Compared to GRPO, our method produces a more concentrated length distribution at each difficulty level. Moreover, shifts in the median across difficulties indicate that DDPO effectively mitigates the \textit{overconfidence} phenomenon, yielding a negative correlation between output length and accuracy.

\textbf{More Analysis}. We further analyzed DDPO from the following aspects: i) To validate the effectiveness of applying $z_i = L_i / L_\text{max}$ when $\text{diff}_q=0$ in~\cref{sec:optimal_length}, 
we compare this with other normalization variants and provide an detailed analysis in Appendix~\ref{app:acc=0}. ii) We conduct additional evaluations on general benchmarks to further assess the generalization ability of DDPO, as reported in Appendix~\ref{app:out-of-domain}. iii) We perform a hyperparameter analysis of the boundary $\theta$ between simple and difficult problems in Appendix~\ref{app:alpha}.



\begin{figure}[t]
\begin{center}
\centerline{\includegraphics[width=0.5\textwidth]{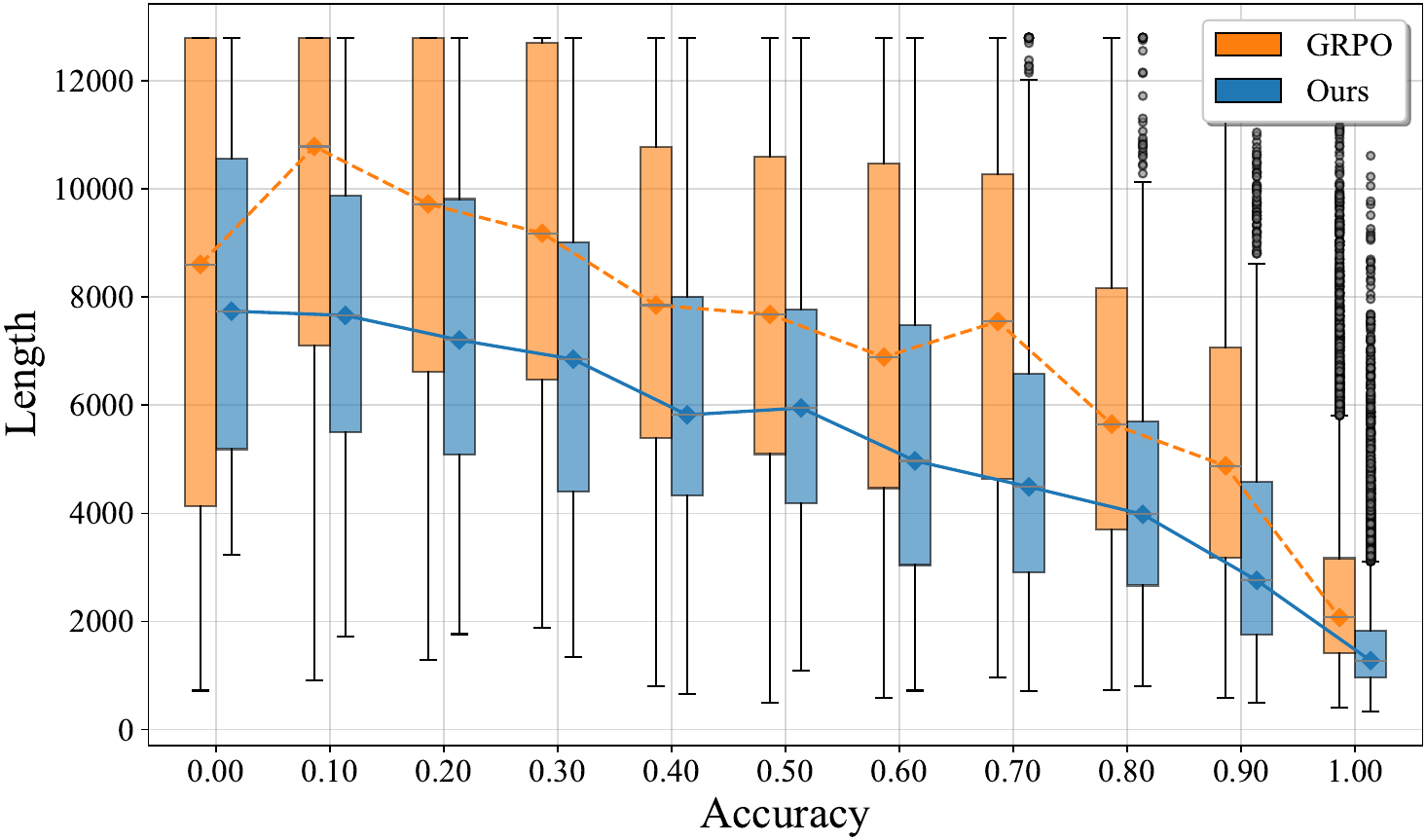}}
\caption{A comparison of the length distributions for different difficulty levels between GRPO and DDPO. The result shows that DDPO results in a more concentrated length distribution and significantly alleviates the \textit{overconfidence} phenomenon.} \label{fig:box}
\vspace{-3em}
\end{center}
\end{figure}

\section{Conclusion}
LRMs often suffer from overthinking, leading to unnecessary computation and potential performance degradation. By analyzing the relationship between output length and accuracy, we identify an \textit{overconfidence} phenomenon: models tend to provide short but incorrect answers for difficult problems, while generating overly long and redundant responses for simpler ones. To address this, we propose Difficulty-Differentiated Policy Optimization (DDPO), which optimizes problems of varying difficulty from distinct perspectives to balance accuracy and length. We provide a theoretical analysis demonstrating that aligning the output length distribution with an optimal target can enhance expected accuracy. Additionally, we introduce a normalization scheme to guide outputs toward this target and approximate the optimal length by the average length of correct answers at the same difficulty level. Extensive experiments across both in-domain and out-of-domain benchmarks demonstrate that our method consistently achieves an improved trade-off between accuracy and output length.









\section*{Impact Statement}
This work introduces a novel method for optimizing reasoning length in large language models (LLMs), addressing the overthinking and overconfidence issues that hinder model performance. By deriving theoretical conditions for expected accuracy and proposing a difficulty-differentiated optimization strategy, our method not only improves model efficiency but also enhances accuracy on both in-domain and out-of-domain benchmarks. This contribution has the potential to advance the design of more efficient, interpretable LLMs, opening new avenues for future research in natural language processing and reinforcement learning.




\nocite{langley00}

\bibliography{reference.bib}
\bibliographystyle{icml2026}

\newpage
\appendix
\onecolumn
\section*{Appendix}

\section{Analysis of the Optimal Length when $\text{diff}_q=0$} \label{app:acc=0}
For extremely difficult problems with an accuracy of 0, there is no optimal reference length to reference, so we aim to encourage the model to generate longer responses to maximize exploration. In~\cref{sec:optimal_length}, when $\text{diff}_q=0$, we design the length normalization function as $z_i=L_i/L_\text{max}$ to encourage exploration as much as possible. 
To validate the reasonability and effectiveness of the setting, we compared the other variation of the normalization function: i) $\mu_\text{similar}$, the average length of the correct answer for questions with the most similar difficulty. ii) $\mu_\text{batch}$, the average length of the correct answer across all questions in the batch. iii) $\mu_q$, the average question length, with no right answer. iv) ZERO, which normalizes the length by $z_i=L_i/L_\text{max}$. As shown in~\cref{tab:acc=0}, when $acc=0$, normalizing by $z_i=L_i/L_\text{max}$ achieves the best performance. Any normalization method that subtracts a mean length restricts the exploration space for problems with $\text{diff}_q=0$, limiting the model's accuracy.
In particular, this leads to a reduction in the overall response length, demonstrating that our method can adaptively converge to a globally shorter response length.

\begin{table*}[h]
\caption{Analysis of optimal length approximation when $acc=0$ ($\theta=0.2$). When $acc=0$, the exploration space should be maximized. The results show that normalizing by $z_i = L_i / L_\text{max}$ yields the best performance, as indicated by the \textbf{bold} results.
}\label{tab:acc=0}
\vspace{-0.8em}
\renewcommand\arraystretch{1.15}
\fontsize{6}{7}\selectfont\setlength{\tabcolsep}{1.5mm}
\begin{center}
\begin{small}
\begin{sc}
\begin{tabular}{lcccccccccccc|cc}

\toprule
\multirow{2}{*}{AVG-TYPE}
 & \multicolumn{2}{c}{\textbf{OLYMPIAD}}                         & \multicolumn{2}{c}{\textbf{MATH}}                                    & \multicolumn{2}{c}{\textbf{AMC}}                                     & \multicolumn{2}{c}{\textbf{AIME2025}}                                & \multicolumn{2}{c}{\textbf{AIME2024}}                               & \multicolumn{2}{c}{\textbf{GPQA-D}}                                & \multicolumn{2}{c}{\textbf{AVG.}}   \\
 & acc.& len. & acc. & len.& acc. & len.& acc. & len.& acc. & len.& acc. & len.& acc. & len. \\
\midrule
$\mu_
\text{similar}$ &$54.81$&$5308 $&$91.10 $&$2590  $&$\textbf{89.84}$&$4182 $&$43.75$&$8957$&$60.21$&$8104 $&$57.87 $&$5522  $&$66.26$&$5777  $ \\

$\mu_\text{batch}$
 & $54.96$ &$5220 $&$91.45$&$2352  $&$89.22$&$4067 $&$42.92$&$9375  $&$59.38 $&$7900  $&$57.46$&$5472  $&$65.90$&$5731  $ \\

$\mu_q$
 & $54.45$ & $5361 $ & $90.76$ & $2596 $ & $89.22$ & $4065 $ & $43.96$ & $\textbf{8382} $ & $57.50$ & $7797 $ & $57.55$ & $5620 $ & $65.57$ & $5637$ \\

zero
 & $\textbf{55.42}$ & $\textbf{5028}$ & $\textbf{92.05}$ & $\textbf{2165}$ & $\textbf{89.84}$ & $\textbf{4056}$ & $\textbf{44.17}$ & $9204$ & $\textbf{61.46}$ & $\textbf{7119}$ & $\textbf{57.80}$ & $\textbf{5390}$ & $\textbf{66.79} $ & $\textbf{5494}$ \\
\bottomrule
\end{tabular}
\end{sc}
\end{small}
\end{center}
\end{table*}

\section{Performance on Other Out-of-Domain Benchmarks.} \label{app:out-of-domain}
To further validate the performance of DDPO ($\theta=0.2$) on out-of-domain data, we conducted additional evaluations on three general benchmarks: MathQA, MMLU-Pro, and BBEH. These benchmarks assess the robustness of our method in diverse real-world scenarios. The results in~\cref{tab:out-of-domain} show that, compared to GRPO, our method achieves significant improvements in length reduction across all three datasets, while maintaining or even enhancing the accuracy. This demonstrates the superior generalization ability of our approach, indicating its strong potential for handling a wide range of tasks beyond the training domain. Specifically, the average length of DDPO is reduced by 28.81\% compared to GRPO, significantly making the output more concise.

\begin{table}[h]
\caption{Comparison with GRPO on out-of-domain benchmarks based on Qwen3-14b, results demonstrate that DDPO outperforms GRPO both in terms of average accuracy and output length.
}\label{tab:out-of-domain}
\vspace{-0.8em}
\renewcommand\arraystretch{1.1}
\fontsize{7}{8}\selectfont\setlength{\tabcolsep}{2.9mm}
\begin{center}
\begin{small}
\begin{sc}
\begin{tabular}{lcc|cc}

\toprule
\multirow{2}{*}{Benchmarks}
& \multicolumn{2}{c}{\textbf{ACC.}} & \multicolumn{2}{c}{\textbf{LEN.}}   \\
 & grpo & ddpo & grpo & ddpo \\
\midrule
math\_qa &$86.15$ & $\textbf{86.46}$& $2555$ & $\textbf{2313}$\\

bbeh
 &$21.03$ &$\textbf{21.20}$&$6851$ & $\textbf{6185}$ \\

mmlu\_pro
 & $\textbf{17.58}$ &$17.41$ & $1817$&$\textbf{1548}$ \\
avg.
 &$41.59$ & $\textbf{41.69}$ &$4703$ & $\textbf{3348}$ \\
\bottomrule
\vspace{-8mm}
\end{tabular}
\end{sc}
\end{small}
\end{center}
\end{table}

\section{Effect of the Threshold $\theta$} \label{app:alpha}
We performed a hyperparameter analysis on $\theta$, the boundary between simple and difficult problems. As shown in~\cref{tab:alpha}, an increase in $\theta$ results in a growing number of problems being classified as difficult, which in turn leads to higher length penalties for the model. Consequently, the average response length also increases, but the accuracy does not improve accordingly. By referring to~\cref{fig:line}, we observe that the phenomenon of \textit{overconfidence} appears to be demarcated by the range of 0.2 to 0.4. Therefore, excessively large values of $\theta$ cause many simple problems to be misclassified as difficult, and excessively long responses result in a decline in accuracy. When $\theta$ is adjusted from 0.2 to 0.4, we observe a significant improvement in accuracy for AIME2025, from $44.17\%$ to $47.08\%$, indicating that more challenging problems in AIME2025 require a larger search space to be effectively explored. When $\theta$ is 0.2, the accuracy of MATH is highest, indicating that there is significant room for length reduction in the answers for this benchmark. These findings underscore the importance of carefully tuning $\theta$ to balance problem classification and length optimization for improved overall performance.

\begin{table*}[h]
\caption{Analysis of $\theta$. The performance of $\theta=0.2$ and $\theta=0.4$ demonstrates superior results, corresponding to the boundary where the \textit{overconfidence} phenomenon is observed in~\cref{fig:line}
The \textbf{bold} represents the best result. 
}\label{tab:alpha}
\vspace{-0.8em}
\renewcommand\arraystretch{1.15}
\fontsize{6}{7}\selectfont\setlength{\tabcolsep}{1.5mm}
\begin{center}
\begin{small}
\begin{sc}
\begin{tabular}{lcccccccccccc|cc}

\toprule
\multirow{2}{*}{Parameters}
 & \multicolumn{2}{c}{\textbf{OLYMPIAD}}                         & \multicolumn{2}{c}{\textbf{MATH}}                                    & \multicolumn{2}{c}{\textbf{AMC}}                                     & \multicolumn{2}{c}{\textbf{AIME2025}}                                & \multicolumn{2}{c}{\textbf{AIME2024}}                               & \multicolumn{2}{c}{\textbf{GPQA-D}}                                & \multicolumn{2}{c}{\textbf{AVG.}}   \\
  & acc.& len. & acc. & len.& acc. & len.& acc. & len.& acc. & len.& acc. & len.& acc. & len. \\
\midrule
$\theta=0.2$ &$55.42$ & $\textbf{5028}$ & $\textbf{92.05}$ & $\textbf{2165}$ &$89.84$&$4057$&$44.17$&$9205 $&$\textbf{61.46}$&$\textbf{7120}  $&$\textbf{57.80}$&$\textbf{5390}   $&$66.79$&$\textbf{5494}  $ \\

$\theta=0.4$
 & $\textbf{55.70}$ &$5299  $&$91.84$&$2321   $&$\textbf{91.09}$&$\textbf{3622}  $&$\textbf{47.08}$&$\textbf{8359} $&$60.63 $&$8074 $&$57.68$&$5683   $&$\textbf{67.34}$&$5560$ \\

$\theta=0.6$
 & $54.51$ & $5551 $ & $90.83$ & $2674  $ & $88.44$ & $4193  $ & $44.38$ & $8839 $ & $59.38$ & $8110  $ & $56.44$ & $6157 $ & $65.66$ & $5921 $ \\

$\theta=0.8$
 & $54.93$ & $5507$ & $91.56$ & $2570 $ & $90.00$ & $4427$ & $44.58$ & $9257$ & $60.21 $ & $7990 $ & $56.79$ & $5919$ & $66.34$ & $5945 $ \\
\bottomrule
\end{tabular}
\end{sc}
\end{small}
\end{center}
\end{table*}

\section{Ablation Study} \label{app:ablation}
We systematically evaluated the importance of each mechanism in~\cref{tab:ablation} based on the fine-tuned Qwen3-14B model, with $\theta=0.4$.~\cref{tab:ablation_full} presents the complete results of the ablation study, where each component contributes to a significant improvement in model accuracy, highlighting the effectiveness of each individual component. 

\begin{table*}[h]
\caption{Ablation study. Each line validates the effectiveness of a specific component, with the final line representing the complete DDPO method. 
The {\color[HTML]{C00000}\textbf{red}} and {\color[HTML]{4472C4}\textbf{blue}} represent the best and second-best result respectively. 
}\label{tab:ablation_full}
\vspace{-0.8em}
\renewcommand\arraystretch{1.15}
\fontsize{6}{7}\selectfont\setlength{\tabcolsep}{1.5mm}
\begin{center}
\begin{small}
\begin{sc}
\begin{tabular}{lcccccccccccc|cc}

\toprule
\multirow{2}{*}{Components}
 & \multicolumn{2}{c}{\textbf{OLYMPIAD}}                         & \multicolumn{2}{c}{\textbf{MATH}}                                    & \multicolumn{2}{c}{\textbf{AMC}}                                     & \multicolumn{2}{c}{\textbf{AIME2025}}                                & \multicolumn{2}{c}{\textbf{AIME2024}}                               & \multicolumn{2}{c}{\textbf{GPQA-D}}                                & \multicolumn{2}{c}{\textbf{AVG.}}   \\
 & acc.& len. & acc. & len.& acc. & len.& acc. & len.& acc. & len.& acc. & len.& acc. & len. \\
\midrule
$L_i/L_\text{max}$ &$53.80$ & \color[HTML]{C00000}$\textbf{4955}$ & $90.65$ & \color[HTML]{C00000}$\textbf{1979}$ & $87.81$ & $4859 $ & $42.29$ & $9248$ & $56.67$ & $8078$ & $55.84$ & $6168$ & $64.51$ & $5881$ \\
$\alpha$
 & $54.00$ & $5450$ & $90.28$ & \color[HTML]{4472C4}$\textbf{2165}$ & $88.13$ & \color[HTML]{4472C4}$\textbf{3822}$ & $43.33$ & $9018$ & $57.71$ & \color[HTML]{C00000}$\textbf{8062}$ & \color[HTML]{4472C4}$\textbf{56.69}$ & \color[HTML]{C00000}$\textbf{5636}$ & $65.02$ & \color[HTML]{4472C4}$\textbf{5692}$ \\

$\mu_{\text{diff}_q}$
 & $54.57$ & $5649$ & $90.95$ & $2512$ & $89.06$ & $4148$ & $44.17$ & \color[HTML]{4472C4}$\textbf{8880}$ & $57.92$ & $8209$ & $56.12$ & $6410$ & $65.47$ & $5968$ \\

LB
 & \color[HTML]{4472C4}$\textbf{54.81}$ & $6068$ & \color[HTML]{4472C4}$\textbf{91.38}$ & $2672$ & \color[HTML]{4472C4}$\textbf{89.91}$ & $4359$ & \color[HTML]{4472C4}$\textbf{44.58}$ & $9042$ & \color[HTML]{4472C4}$\textbf{58.67}$ & $8268$ & $56.34$ & $6701$ & \color[HTML]{4472C4}$\textbf{65.95}$ & $6185$ \\

UB
& \color[HTML]{C00000}$\textbf{55.70}$ & \color[HTML]{4472C4}$\textbf{5299}$ & \color[HTML]{C00000}$\textbf{91.84}$ & $2321$ & \color[HTML]{C00000}$\textbf{91.09}$ & \color[HTML]{C00000}$\textbf{3622}$ & \color[HTML]{C00000}$\textbf{47.08}$ & \color[HTML]{C00000}$\textbf{8359}$ & \color[HTML]{C00000}$\textbf{60.63}$ & \color[HTML]{4472C4}$\textbf{8074}$ & \color[HTML]{C00000}$\textbf{57.68}$ & \color[HTML]{4472C4}$\textbf{5683}$ & \color[HTML]{C00000}$\textbf{67.34}$ & \color[HTML]{C00000}$\textbf{5560}$ \\
\bottomrule
\end{tabular}
\end{sc}
\end{small}
\end{center}
\end{table*}

\section{Performance on other Baselines} \label{app:baselines}
We provided the full results on six benchmarks in~\cref{tab:baseline_full}. For GRPO as the baseline, our method improves accuracy across all benchmarks while reducing response length. For DrGRPO, our method enhances accuracy on all benchmarks and reduces length on most datasets. For DAPO, since it already incorporates penalties for long answers, there may be some conflict with the DDPO mechanisms, leading to an increase in response length. However, the average accuracy still improves. Overall, these results demonstrate that our method effectively balances accuracy and length across different baselines and datasets.

\begin{table*}[h]
\caption{The effectiveness of our method ($\theta=0.2$) across different baselines. The {\textcolor{gray}{gray}} background indicates the performance of our method, adding on varying baselines; \textbf{Bolded} values denote the better performance between baseline and our method.
}\label{tab:baseline_full}
\vspace{-0.8em}
\renewcommand\arraystretch{1.2}
\fontsize{6}{7}\selectfont\setlength{\tabcolsep}{1.3mm}
\begin{center}
\begin{small}
\begin{sc}
\begin{tabular}{lcccccccccccccc}

\toprule
\multirow{2}{*}{Baselines}
 & \multicolumn{2}{c}{\textbf{OLYMPIAD}}                         & \multicolumn{2}{c}{\textbf{MATH}}                                    & \multicolumn{2}{c}{\textbf{AMC}}                                     & \multicolumn{2}{c}{\textbf{AIME2025}}                                & \multicolumn{2}{c}{\textbf{AIME2024}}                               & \multicolumn{2}{c}{\textbf{GPQA-D}}                                & \multicolumn{2}{c}{\textbf{AVG.}}   \\
 & acc. & len. & acc. & len.& acc. & len.& acc. & len.& acc. & len.& acc. & len.& acc. & len. \\
\midrule
DAPO &$55.44$&$5149$&$91.61 $&$\textbf{2190} $&$89.53 $&$\textbf{3781}$&$43.54 $&$\textbf{7805}$&$\textbf{61.46}$&$\textbf{6402}$&$57.61 $&$\textbf{4194}$&$66.53$&$\textbf{4920}$ \\

\rowcolor{gray!12}
DAPO+ours
 & $\textbf{56.13} $ &$\textbf{5145}$&$\textbf{92.05} $&$2466$&$\textbf{89.84} $&$3899$&$\textbf{45.83 }$&$8811$&$59.38 $&$7777$&$\textbf{58.11} $&$5280 $&$\textbf{66.89} $&$5563$ \\
 
DrGRPO
 & $54.31$ & $6171$ & $91.50$ & $3117$ & $89.22$ & $5089$ & $42.29$ & $9546$ & $\textbf{58.33}$ & $8487$ & $55.68$ & $6760$ & $65.22 $ & $6528$ \\
 
\rowcolor{gray!12}
 DrGRPO+ours
 & $\textbf{55.00}$ & $\textbf{5334}$ & $\textbf{91.60}$ & $\textbf{2518}$ & $\textbf{89.38}$ & $\textbf{4113}$ & $\textbf{45.00}$ & $\textbf{9145}$ & $58.13$ & $\textbf{7923}$ & $\textbf{56.85}$ & $\textbf{5913}$ & $\textbf{65.99} $ & $\textbf{5824}$ \\

 \text{GRPO} 
 & $55.39$ & $5885$ & $91.86$ & $3026$ & $\textbf{89.84}$ & $4080$ & $42.71$ & $9569$ & $56.04$ & $8927$ & $57.11$ & $5728$ &  $65.49$ & $6324$ \\
 
\rowcolor{gray!12}
 \text{GRPO+Ours} 
 & $\textbf{55.42}$ & $\textbf{5027}$ & $\textbf{92.05}$ & $\textbf{2165}$ & $\textbf{89.84}$ & $\textbf{4056}$ & $\textbf{44.17}$ & $\textbf{9204}$ & $\textbf{61.46}$ & $\textbf{7119}$ & $\textbf{57.80}$ & $\textbf{5390}$ & $\textbf{66.79} $ & $\textbf{5494}$ \\
\bottomrule
\end{tabular}
\end{sc}
\end{small}
\end{center}
\end{table*}


\end{document}